%% file: main.tex
\pdfoutput=1
\documentclass[11pt]{article}

\input{macros}

\usepackage{tikz}
\usepackage{pgfplots}

\usetikzlibrary{arrows.meta}
\usetikzlibrary{decorations.pathreplacing}

\title{The Tractability Landscape of Sampling with Inexact Scores}
\author{
Anming Gu\thanks{University of Texas at Austin, \texttt{anminggu@cs.utexas.edu}}
\and 
Kevin Tian\thanks{University of Texas at Austin, \texttt{kjtian@cs.utexas.edu}}
\and
Hubert Yang\thanks{University of Texas at Austin, \texttt{hubert.yang@utexas.edu}}
\and
Yusong Zhu\thanks{University of Texas at Austin, \texttt{zhuys@utexas.edu}}
}
\date{}

\begin{document}
\allowdisplaybreaks
\maketitle

\begin{abstract}
We provide a simple and tight characterization of the types of inexact score oracle access that permit sampling with vanishing total variation bias, for a standard, well-behaved target family. Our main result shows that any weaker error than the sub-Gaussian assumption used by \cite{YangW26} rules out the tractability of unbiased sampling. This strengthens the conclusion of  \cite{CaoCSW26} to be \emph{algorithm-agnostic}, and to hold for a wider range of error assumptions.
\end{abstract}

\section{Introduction}

The development of an \emph{inexact oracle complexity theory} has been central to our understanding of stochastic optimization \cite{NemirovskyY83, AgarwalBRW12}. We consider an analogous question in sampling.

\begin{problem}\label{prob:main}
    Given a distribution $\pi$ on $\R^d$ belonging to a well-behaved class, and oracle access to $\vs: \R^d \to \R^d$ 
    that approximates $\vs_\pi \defeq \nabla \log \pi$, under what error assumptions on $(\vs, \pi)$ is it possible to sample to within arbitrary TV distance of $\pi$, as we take the error bound to $0$?
\end{problem}

We focus on the canonical family of $\pi$ that are log-smooth (i.e., $\vs_\pi$ is Lipschitz), and that satisfy
a \emph{log-Sobolev inequality} (LSI).\footnote{The LSI is a robust form of isoperimetry that is often taken to admit tractability of sampling algorithms \cite{VempalaW19}, and applies to (but is not limited to) strongly log-concave distributions \cite{BakryE85}.} For this target family, \cite{YangW26} shows that a \emph{sub-Gaussian score error} (i.e., $\beta = 2$ in Definition~\ref{def:mgf_error}) with appropriate parameters is sufficient to solve Problem~\ref{prob:main}.

Our main finding is that the error condition assumed by \cite{YangW26} is necessary: any qualitatively weaker condition results in intractability, even for log-smooth, LSI targets.

\begin{theorem}[Informal, see Theorem~\ref{thm:main_full}]\label{thm:main}
Standard weaker variants of the assumption used in \cite{YangW26}, e.g., bounded moments (Definition~\ref{def:lp_error}) or sub-Weibull error (Definition~\ref{def:mgf_error} with $\beta \in (0, 2)$) make sampling to vanishing $\TV$ impossible, even as the error bound goes to $0$. Moreover, the parameters used in the sub-Gaussian error assumption in \cite{YangW26} are also tight up to constant factors.
\end{theorem}

Our counterexample is simple and geometric, and is illustrated in Figure~\ref{fig:illustration}. It uses two well-separated Gaussians, and an estimate $\vs$ that interpolates between the true scores. The LSI property allows both distributions to simultaneously satisfy any weaker score error than sub-Gaussianity.

Theorem~\ref{thm:main} is related to recent work of \cite{CaoCSW26}, who considered the tractability of Problem~\ref{prob:main} under score estimation error with bounded moments. Our result strengthens their lower bound in the following ways. First, it applies to \emph{any algorithm} querying $\vs$, whereas \cite{CaoCSW26} only disproved convergence of the inexact Langevin dynamics. Second, it rules out a wider range of inexactness notions, i.e., essentially the entire range of error assumptions below that used by \cite{YangW26}.

\textbf{Preliminaries.}
Let $\ind_\calE$ denote the $0$-$1$ indicator of an event $\calE$.
Let $\calP(\Omega)$ denote the set of probability measures on $\Omega$. We only consider measures in $\calP(\R^d)$ that are absolutely continuous with respect to the Lebesgue measure, and we overload notation by identifying a measure with its density function. For $\mu,\pi \in \calP(\R^d)$, we denote their total variation distance by $\TV(\mu, \pi) \defeq \half \int_{\R^d} |\mu(\vx) - \pi(\vx)| \dd \vx$.
For $\pi \in \calP(\R^d)$, we let $\vs_\pi \defeq \nabla \log \pi$ be the \emph{true score}. We use $\vs: \R^d \to \R^d$ to denote a vector field, often an approximate score.
We say that $\pi \in \calP(\R^d)$ is \emph{$L$-log-smooth} and \emph{$\mu$-strongly log-concave} if, respectively, $-\log \pi$ is $L$-smooth and $\mu$-strongly convex (cf.\ Sections 3.2, 3.4 \cite{Bubeck15}). If $\pi$ is $\mu$-strongly log-concave, it also satisfies a $\mu$-LSI.

Next, we define the two types of score error assumptions we consider.
\begin{definition}[$L^p$ error]\label{def:lp_error}
For $\vs: \R^d \to \R^d$, $\pi \in \calP(\R^d)$, $\eps > 0$, and $p > 0$, if \[
\E_\pi\Brack{\norm{\vs - \vs_\pi}_2^p}^{1/p}\le \eps,\]
    we say that the pair $(\vs, \pi)$ satisfies an \emph{$\eps$-$L^p$  error}.
\end{definition}

\begin{definition}[MGF error]\label{def:mgf_error}
For $\vs: \R^d \to \R^d$, $\pi \in \calP(\R^d)$, $\eps > 0$, $r > 0$, and $\beta > 0$, if
\[
\Par{\frac1r\log\E_{\pi}\Brack{\exp\Par{r\norm{\vs - \vs_\pi}_2^\beta}}}^{1/\beta}\le \eps,\]
we say that the pair $(\vs, \pi)$ satisfies an \emph{$\eps$-$(\beta, r)$-MGF  error}.
\end{definition}

The following result shows that MGF error implies $L^p$ error (with appropriate parameters), so that hard instances satisfying Definition~\ref{def:mgf_error} also act as lower bounds under Definition~\ref{def:lp_error}.

\begin{lemma}\label{lem:mgf_to_lp}
Fix $\eps > 0$ and $p > 0$. There exist $\beta \in (0, 2)$ and $r > 0$ such that if the pair $(\vs, \pi)$ satisfies an $\eps$-$(\beta, r)$-MGF error, then it also satisfies an $\eps$-$L^p$ error.
\end{lemma}
\begin{proof}
Let $\beta = 1$ and $r = \frac p \eps$, and assume that $(\vs, \pi)$ satisfies an $\eps$-$(\beta, r)$-MGF error. 

For every $x > 0$, we have $\log x \le x - 1$, which rearranged gives $x^p \le \exp(-p + px)$. Applying this pointwise with $x \gets \frac 1 \eps \norm{\vs - \vs_\pi}_2$ and taking expectations over $\pi$,
\[\E_\pi\Brack{ \Par{\frac{\norm{\vs - \vs_\pi}_2}{\eps}}^p} \le \exp(-p)\E_\pi\Brack{\exp\Par{r \norm{\vs - \vs_\pi}_2}} \le \exp(-p)\exp(r\eps) = 1, \]
where the last inequality used our MGF error assumption. Rearranging the above display and taking $p^{\text{th}}$ roots now shows that $(\vs, \pi)$ satisfies an $\eps$-$L^p$ error.
\end{proof}

We also observe that by Jensen's inequality, Definition~\ref{def:lp_error} is monotone nondecreasing in strength as $p$ grows, and Definition~\ref{def:mgf_error} is monotone nondecreasing in strength as $r$ grows.

\section{Lower Bounds}\label{sec:proof}

\begin{figure}[t]
\centering
\begin{tikzpicture}[
    scale=0.95,
    label box/.style={
        fill=white,
        fill opacity=0.94,
        text opacity=1,
        rounded corners=1pt,
        inner sep=2pt
    }
]

\def\R{6}
\def\gaussnorm{2.39894228} 


\fill[gray!10]
    (2,-3.0) rectangle (4,3.0);

\draw[dashed,gray]
    (2,-3.0) -- (2,3.0);

\draw[dashed,gray]
    (4,-3.0) -- (4,3.0);

\node[
    label box,
    text=gray,
    anchor=south
] at (3,3.08)
    {\small score interpolation};


\draw[->] (-3.2,0) -- (9.2,0)
    node[right] {$x$};

\draw[->] (0,-3.2) -- (0,3.3);


\draw[blue!70,very thick]
plot[smooth,domain=-3:3,samples=120]
    (\x,{\gaussnorm*exp(-\x*\x/2)});

\draw[red!70,very thick]
plot[smooth,domain=3:9,samples=120]
    (\x,{\gaussnorm*exp(-(\x-\R)*(\x-\R)/2)});

\node[
    label box,
    text=blue!70,
    anchor=east
] at (-0.35,-0.55)
    {$\pi_{1,\epsilon}=\mathcal N(0,1)$};

\node[
    label box,
    text=red!70,
    anchor=west
] at (7.35,1.38)
    {$\pi_{2,\epsilon}=\mathcal N(R,1)$};


\draw[blue!70,dashed,thick]
plot[domain=-3:2,samples=2]
    (\x,{-\x});

\draw[black,dashed,thick]
plot[smooth,domain=2:4,samples=120]
    (\x,{
        -\x
        + \R*(0.5-0.5*cos(180*(\x-2)/2))
    });

\draw[red!70,dashed,thick]
plot[domain=4:9,samples=2]
    (\x,{-(\x-\R)});


\node[
    label box,
    text=blue!70,
    anchor=west
] at (-3.0,3.05)
    {$s_{\pi_{1,\epsilon}}(x)=-x$};

\node[
    label box,
    text=red!70,
    anchor=west
] at (4.15,-2.45)
    {$s_{\pi_{2,\epsilon}}(x)=-(x-R)$};

\node[
    label box,
    anchor=east
] at (3.85,-0.98)
    {$s_\epsilon(x)$};


\node[
    label box,
    text=gray,
    anchor=north
] at (2,-0.10)
    {$R/3$};

\node[
    label box,
    text=gray,
    anchor=north
] at (4,-0.10)
    {$2R/3$};

\end{tikzpicture}

\caption{
Scores only disagree with $s_\epsilon$ on regions where
corresponding measures assign small mass.
}
\label{fig:illustration}
\end{figure}
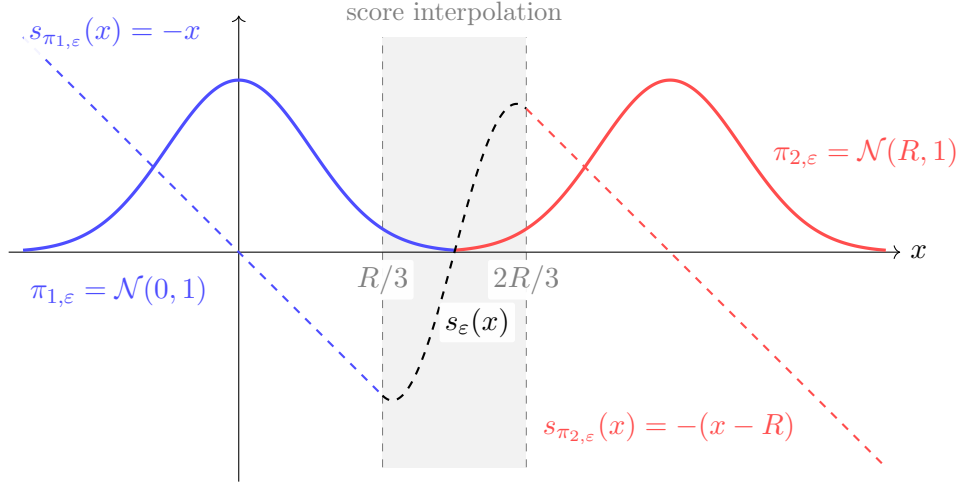

We begin by providing our main hard example.

\begin{proposition}\label{prop:main_full}
Fix $\beta \in (0, 2)$ and $r > 0$. For any $\eps > 0$, there exist two $1$-log-smooth, $1$-strongly log-concave distributions
     $\pi_{1,\eps}, \pi_{2,\eps} \in \calP(\R^d)$, and $5$-Lipschitz  $\vs_\eps:\R^d\to\R^d$, such that $(\vs_\eps, \pi_{1, \eps})$ and $(\vs_\eps, \pi_{2, \eps})$ both satisfy $\eps$-$(\beta, r)$-MGF error, and $\TV(\pi_{1, \eps}, \pi_{2, \eps}) > \frac 2 3$.\footnote{We use $\frac 2 3$ for simplicity, but any constant $< 1$ is achievable through a small modification of the proof.}
\end{proposition}

\begin{proof}
Our example is in dimension $d = 1$, so we use $s_\eps$ instead of $\vs_\eps$ to denote the scalar approximate score.
Choose $R > 2$ to be any value such that  \begin{equation}
\Par{\frac1r\log\Par{1+\frac{3}{R\sqrt{2\pi}}\exp\Par{-\frac{R^2}{18}+rR^\beta}}}^{1/\beta}\le \eps.    \label{eq:bound_eps}
\end{equation}
Such a value exists as $\beta <2$. We now take $\pi_{1,\eps}=\calN(0,1)$ and $\pi_{2,\eps}=\calN(R,1)$. It is immediate that both distributions are 1-strongly log-concave and 1-smooth, and that $\TV(\pi_{1, \eps}, \pi_{2, \eps}) \ge \frac 2 3$.

It remains to construct $s_\eps$. We first define 
\[\rho(t)\defeq \exp\Par{-\frac 1 t}\ind_{t > 0}, \quad \theta(t) \defeq \frac{\rho(t)}{\rho(t)+\rho(1-t)},\quad \chi(u) = \theta(3u-1).\] Observe that $\theta \in C^\infty(\R)$, $\theta(t) = 0$ for $t\le 0$ and $\theta(t) = 1$ for $t\ge 1$. We next define 
\begin{equation}
    s_\eps(x)\defeq -x+R\cdot\chi\Par{\frac x R},\label{eq:score}    
    \end{equation}
    
    which we note agrees with the scores of $\pi_{1,\eps}$ and $\pi_{2,\eps}$ on $x\le \frac{R}{3}$ and $x\ge \frac{2R}{3}$, respectively. 

    We now bound the Lipschitz constant of $s_\eps$. For $t\in(0, 1)$, we can calculate \[
    \theta'(t) = \frac{\rho(t)\rho(1-t)}{(\rho(t)+\rho(1-t))^2}\Par{\frac{1}{t^2}+\frac{1}{(1-t)^2}}.
    \]
    Outside this interval, $\theta'=0$. Then, we can check that $\theta'$ is maximized at $t=\frac{1}{2}$ by observing $\theta''\ge 0$ over $0\le t \le \frac{1}{2}$ and $\theta'$ is symmetric on $(0, 1)$. This ensures $\theta' \in [0, 2]$, which implies $s'_\eps \in [-1, 5]$.

    We conclude by proving $(s_\eps, \pi_{1, \eps})$ satisfy $\eps$-$(\beta, r)$-MGF error, and the same claim holds for $\pi_{2, \eps}$ by symmetry. The true score is $s_{\pi_{1, \eps}}(x) = -x$ and $\chi$ has range $[0, 1]$, so by \eqref{eq:score},
    \[\Delta = \sup |s_\eps - s_{\pi_{1, \eps}}| \le R.\]
    Moreover, since $\chi(u) = 0$ if $u \le \frac 1 3$, we know that $s_\eps(x) \neq s_{\pi_{1, \eps}}(x)$ implies $x > \frac R 3$. Thus,
    \begin{align*}
    \E_{x \sim \pi_{1, \eps}}\Brack{\exp\Par{r\Abs{s_\eps(x) - s_{\pi_{1, \eps}}(x)}^\beta}} &\le 1 + \E_{x \sim \pi_{1, \eps}}\Brack{\Par{\exp\Par{r\Delta^\beta} - 1}\ind_{x > \frac R 3}} \\
    &\le 1 + \exp(rR^\beta) \Pr_{x \sim \pi_{1, \eps}}\Brack{x > \frac R 3} \\
    &\le 1 + \frac{3}{R\sqrt{2\pi}}\exp\Par{-\frac{R^2}{18}+rR^\beta}.
    \end{align*}
    The last line applied a standard Gaussian tail bound (e.g., Proposition 2.1.2, \cite{Vershynin25}). Thus we have shown that $(s_\eps, \pi_{1, \eps})$ satisfies $\eps$-$(\beta, r)$-MGF error, upon using the definition \eqref{eq:bound_eps}.
\end{proof}

\begin{corollary}\label{cor:small_r}
The conclusion of Proposition~\ref{prop:main_full} also holds when $\beta = 2$, $r < \frac 1 {18}$.
\end{corollary}
\begin{proof}
The only place the assumed bounds on $(\beta, r)$ were used in Proposition~\ref{prop:main_full}'s proof was in the existence of $R$ satisfying \eqref{eq:bound_eps}. This existence still holds for $\beta = 2$, $r < \frac 1 {18}$.
\end{proof}

\begin{theorem}\label{thm:main_full}
There does not exist an algorithm $\calA$ with the property that: for any  $1$-log-smooth $\pi$ satisfying $1$-LSI, and any $\vs: \R^d \to \R^d$ satisfying one of the following error assumptions, $\calA$ only makes queries to $\vs$ and returns a sample $\vx$ such that $\TV(\Law(\vx), \pi) \le \frac 1 3$.
\begin{itemize}
    \item $(\vs, \pi)$ satisfies an $\eps$-$(\beta, r)$-MGF error for any $\eps > 0$, where either $\beta \in (0, 2)$, or $\beta = 2$, $r < \frac 1 {18}$.
    \item $(\vs, \pi)$ satisfies an $\eps$-$L^p$ error, for any $\eps > 0$, $p > 0$.
\end{itemize}
\end{theorem}
\begin{proof}
The first claim is immediate from Proposition~\ref{prop:main_full} and Corollary~\ref{cor:small_r}, by taking $\vs \gets \vs_\eps$. Indeed, then $\pi$ could be either $\pi_{1, \eps}$ or $\pi_{2, \eps}$, and the triangle inequality contradicts $\TV(\pi_{1, \eps}, \pi_{2, \eps}) > \frac 2 3$. The second claim follows by applying Proposition~\ref{prop:main_full} with $(\beta, r)$ given by Lemma~\ref{lem:mgf_to_lp}.
\end{proof}

We note that the $\beta = 2$, $r < \frac 1 {18}$ branch ruled out by Theorem~\ref{thm:main_full} shows that the range of $r$ handled by Theorem 2, \cite{YangW26} (roughly, larger than the inverse LSI constant) is also tight up to constants.

\begin{remark}
It suffices to take $R \approx \sqrt{\log(1/\eps)}$ for 
the construction in Proposition~\ref{prop:main_full}. The construction does not extend if we restrict the range on $R$ to be independent of $\eps$. Indeed, no score $s$ can have vanishing error $\eps \to 0$ with respect to both $s_{\pi_{1, \eps}}(x) = -x$ and $s_{\pi_{2, \eps}}(x) = -x + R$:
\begin{align*}\E_{x \sim \Nor(0, 1)}\Brack{\Abs{s(x) + x}^2} + \E_{x \sim \Nor(R, 1)}\Brack{\Abs{s(x) + x - R}^2} \ge \Par{1 - \TV\Par{\Nor(0, 1), \Nor(R, 1)}} \cdot \frac{R^2}{2},
\end{align*}
because $|s + x|^2 + |s + x - R|^2 \ge \frac{R^2}{2}$ for any choice of $s, x$. The right-hand side above cannot be taken $\to 0$ without allowing $R$ to grow. Notably, this shows Theorem~\ref{thm:main_full} does not naturally extend to the \emph{DDPM setting}, where approximate score access with vanishing error is assumed at all convolution scales, including scales where, applied to unit-variance Gaussians, convolving takes $R \to 0$. This is unsurprising, as $L^2$ error is known to suffice for $\TV$ convergence there \cite{ChenCLLSZ23, ChenLL23}.
\end{remark}

\bibliographystyle{alpha}
\bibliography{main}

\end{document}

%% file: macros.tex
\usepackage{fullpage}
\usepackage[english]{babel}
\usepackage[utf8x]{inputenc}
\usepackage[T1]{fontenc}
\usepackage{amsmath}
\usepackage{amssymb}
\usepackage{amsthm}
\usepackage{bbm}
\usepackage{mdframed}
\usepackage{bm}
\usepackage{bbold}
\usepackage{parskip}
\usepackage{thm-restate}
\usepackage{booktabs}
\usepackage{array}
\usepackage[lined,boxed,ruled,norelsize,linesnumbered]{algorithm2e}
\usepackage{hyperref}
\hypersetup{
	colorlinks=true,
	linkcolor=blue!70!black,
	citecolor=blue!70!black,
	urlcolor=blue!70!black
}
\usepackage{cleveref}
\usepackage{graphicx}
\usepackage{subcaption}

\usepackage{mathtools}
\usepackage{xcolor}

\newtheorem{theorem}{Theorem}
\newtheorem{lemma}{Lemma}

\newtheorem{corollary}{Corollary}
\newtheorem{proposition}{Proposition}
\newtheorem{definition}{Definition}

\newtheorem{problem}{Problem}

\newtheorem{remark}{Remark}

\newcommand{\defeq}{:=}

\newcommand{\norm}[1]{\left\lVert#1\right\rVert}

\newcommand{\eps}{\epsilon}

\newcommand{\R}{\mathbb{R}}

\newcommand{\half}{\frac{1}{2}}

\newcommand{\E}{\mathbb{E}}

\newcommand{\Nor}{\mathcal{N}}

\newcommand{\dd}{\textup{d}}

\newcommand{\Par}[1]{\left(#1\right)}
\newcommand{\Brack}[1]{\left[#1\right]}
\newcommand{\Brace}[1]{\left\{#1\right\}}
\newcommand{\Abs}[1]{\left|#1\right|}

\newcommand{\ind}{\mathbb{I}}

\newcommand{\1}{\mathbf{1}}

\newcommand{\vs}{\mathbf{s}}

\newcommand{\vx}{\mathbf{x}}

\newcommand{\calA}{\mathcal{A}}

\newcommand{\calE}{\mathcal{E}}

\newcommand{\calN}{\mathcal{N}}

\newcommand{\calP}{\mathcal{P}}

\renewcommand{\epsilon}{\varepsilon}
\renewcommand{\eps}{\varepsilon}

\newcommand{\TV}{\mathsf{TV}}

\newcommand{\Law}{\textup{Law}}

%% file: main.bib
@misc{Vershynin25,
  title={High-dimensional probability},
  author={Vershynin, Roman},
  year={2025},
  publisher={Cambridge University Press}
}

@article{AgarwalBRW12,
  author       = {Alekh Agarwal and
                  Peter L. Bartlett and
                  Pradeep Ravikumar and
                  Martin J. Wainwright},
  title        = {Information-Theoretic Lower Bounds on the Oracle Complexity of Stochastic
                  Convex Optimization},
  journal      = {{IEEE} Trans. Inf. Theory},
  volume       = {58},
  number       = {5},
  pages        = {3235--3249},
  year         = {2012}
}

@incollection{BakryE85,
  author    = {Bakry, Dominique and {\'E}mery, Michel},
  title     = {Diffusions hypercontractives},
  booktitle = {S{\'e}minaire de Probabilit{\'e}s XIX, 1983/84},
  editor    = {Az{\'e}ma, Jacques and Yor, Marc},
  series    = {Lecture Notes in Mathematics},
  volume    = {19},
  pages     = {177--206},
  publisher = {Springer},
  address   = {Berlin},
  year      = {1985}
}

@inproceedings{CaoCSW26,
  author       = {Cao, Daniel Yiming and Chen, August Y. and Sridharan, Karthik and Wu, Yuchen},
  title        = {On the Robustness of {L}angevin Dynamics to Score Function Error},
  booktitle    = {Forty-third International Conference on Machine Learning, {ICML}
                  2026},
  series       = {Proceedings of Machine Learning Research},
  publisher    = {{PMLR} / OpenReview.net},
  year         = {2026}
}

@inproceedings{ChenCLLSZ23,
title={Sampling is as easy as learning the score: theory for diffusion models with minimal data assumptions},
author={Sitan Chen and Sinho Chewi and Jerry Li and Yuanzhi Li and Adil Salim and Anru Zhang},
booktitle={The Eleventh International Conference on Learning Representations },
year={2023},
url={https://openreview.net/forum?id=zyLVMgsZ0U_}
}

@inproceedings{ChenLL23,
  author       = {Hongrui Chen and
                  Holden Lee and
                  Jianfeng Lu},
  title        = {Improved Analysis of Score-based Generative Modeling: User-Friendly
                  Bounds under Minimal Smoothness Assumptions},
  booktitle    = {International Conference on Machine Learning, {ICML} 2023},
  series       = {Proceedings of Machine Learning Research},
  volume       = {202},
  pages        = {4735--4763},
  publisher    = {{PMLR}},
  year         = {2023}
}

@article{YangW26,
  title={Convergence of the Inexact {L}angevin Algorithm in {KL} Divergence with Application to Score-based Generative Models},
  author={Yang, Kaylee Yingxi and Wibisono, Andre},
  journal={arXiv preprint arXiv:2211.01512},
  year={2026}
}

@article{Bubeck15,
  author       = {S{\'{e}}bastien Bubeck},
  title        = {Convex Optimization: Algorithms and Complexity},
  journal      = {Found. Trends Mach. Learn.},
  volume       = {8},
  number       = {3-4},
  pages        = {231--357},
  year         = {2015}
}

@book{NemirovskyY83,
  author    = {Nemirovsky, A.\ S.\ and Yudin, D.\ B.\},
  title     = {Problem Complexity and Method Efficiency in Optimization},
  series    = {Wiley-Interscience Series in Discrete Mathematics},
  publisher = {John Wiley \& Sons},
  year      = {1983}
}

@article{VempalaW19,
  title={Rapid convergence of the unadjusted {L}angevin algorithm: Isoperimetry suffices},
  author={Vempala, Santosh and Wibisono, Andre},
  journal={Advances in neural information processing systems},
  volume={32},
  year={2019}
}
